\documentclass{article}

\usepackage[preprint]{neurips_2026}
\workshoptitle{AI for Stochastic Dynamics: From Theoretical
  Foundations to Scientific Applications}
\makeatletter
\renewcommand{\@noticestring}{Preprint.}
\makeatother

\usepackage[utf8]{inputenc}
\usepackage[T1]{fontenc}
\usepackage{hyperref}
\usepackage{url}
\usepackage{booktabs}
\usepackage{amsfonts}
\usepackage{dsfont}   
\usepackage{nicefrac}
\usepackage{microtype}
\usepackage{xcolor}
\usepackage{amsmath,amssymb,amsthm}
\usepackage{graphicx}
\usepackage{subcaption}
\usepackage{enumitem}
\usepackage{pgfplots}
\pgfplotsset{compat=1.18}
\usepgfplotslibrary{fillbetween,groupplots}

\newtheorem{definition}{Definition}
\newtheorem{lemma}{Lemma}
\newtheorem{proposition}{Proposition}
\newtheorem{corollary}{Corollary}

\title{Forecasting Multiple Observables with SCROLL:\\
  Score-Trained Uncertainty for Stochastic Dynamics}

\author{%
  Pavel Proch\'azka\\
  Cisco Inc.\\
  \texttt{paprocha@cisco.com}
}

\begin{document}

\maketitle

\begin{abstract}
Forecasting a stochastic dynamical system rarely means a single
number: one wants several observables---future state, threshold event,
regime label---each with its own likelihood. Standard multi-task
recipes balance per-task losses, tuned or learned. We instead compose
the observables' likelihoods in per-task free-routed last-layer
beliefs on a shared backbone; this absorbs unit-dependent loss scaling
into $O(K)$ likelihood parameters learned in the same gradient pass. Stochastic dynamics supply what static
benchmarks cannot: computable ground truth for the predictive
variance. Results land where theory puts them: on the well-specified,
homoscedastic Ornstein--Uhlenbeck process the learned predictive law
recovers the analytic kernel and correctly specified baselines tie.
On heteroscedastic systems (stochastic Lorenz-63, real air-quality
data) the belief's input-dependent variance separates: best
single-run NLL on the state and
regime tasks, calibration matched only by arms whose NLL it beats, at
a fraction of the tuned grids' cost. On the real series the state
margin holds across five rolling origins.
\end{abstract}

\section{Introduction}
\label{sec:intro}


A forecaster attached to a stochastic dynamical system is typically asked
for several observables of the same window at once: the state at lead time
$\Delta$, a barrier-crossing probability, a coarse regime
label. Treated as multi-task learning (MTL) on a shared backbone,
the standard recipes minimise a weighted sum of per-task losses, with
weights set by grid search---exponential in the number of
tasks---or learned as a single scale per task
\citep{kendall2018multi,chen2018gradnorm,yu2020pcgrad}. But weights
only rescale the losses; the choice of each task's \emph{likelihood}
is a separate axis. We exhibit such a failure: a cross-entropy head
that cannot represent the barrier's offset boundary fails at any
weight, tuned or learned, while the threshold likelihood
carries the offset in its cutpoint and repairs it.

We take the Bayesian route instead: each observable keeps its own
observation model and its own Gaussian last-layer belief on one
jointly trained backbone. The construction is SCROLL (Shared-Cavity
fRee-rOuting Last-Layer;~\citealp{prochazka2026bethe}): its loss
keeps one local log-partition term per factor of the model as
written. The joint factor graph thus \emph{derives} the composed
objective---a sum of per-task shared-cavity losses, each belief's
variance free to track state-dependent spread
(Section~\ref{sec:background}).

\textbf{Contributions.}
We propose SCROLL-MT (Definition~\ref{def:fmt}), which is, factor by
factor, the sum of $K$ free-routed single-task SCROLL objectives on
one backbone (Lemma~\ref{lem:compose}). The composition carries the
properties: no outer unit-balancing weight search, the weight slots
filled by $O(K)$ likelihood parameters in one gradient pass at a
fixed equal-per-task convention (Corollary~\ref{cor:weights}); a
population optimum at the score-optimal predictive law
(Corollary~\ref{cor:mt_score}); and scale equivariance without weight
retuning, which a fixed weighted sum of raw losses lacks
(Proposition~\ref{prop:scale}). We validate each on stochastic
differential equation (SDE) corpora with computable conditional
spread $V^\star(x)$ and on a real series, against tuned, learned, and
separately trained baselines.

\section{Free routing and likelihood composition in the last layer}
\label{sec:background}


Data are $N$ windows of the same series: features $\psi_n=\psi(x_n)$
from a shared deterministic backbone, and per window the $K$
observables $y^{(1)}_n,\dots,y^{(K)}_n$. Each task $k$ keeps its own
likelihood $p_k(y^{(k)}\mid w_k^\top\psi;\eta_k)$ with parameters
$\eta_k$, its own belief $q_k(w_k)=\mathcal{N}(\mu_k,\Sigma_k)$ on
task $k$'s last-layer weights $w_k$, and prior
$p_{\alpha_k}(w_k)=\mathcal{N}(0,\alpha_k^{-1}I)$; only the backbone
is shared.

\begin{definition}[SCROLL Multi-Task (SCROLL-MT): free-routing composed objective]
\label{def:fmt}
The objective, minimised jointly over $\psi$ and, per task, over
$(\mu_k,\Sigma_k,\alpha_k,\eta_k)$, is
\begin{equation}
  F_\text{MT}
  \;=\; \sum_{k=1}^{K} \Bigl[
    -\sum_{n=1}^{N} \log \int p_k\bigl(y^{(k)}_n \mid w_k^\top\psi_n;
      \eta_k\bigr)\, q_k(w_k)\, dw_k
    \;-\; \log \int p_{\alpha_k}(w_k)\, q_k(w_k)\, dw_k \Bigr],
  \label{eq:fmt}
\end{equation}
\end{definition}

\begin{lemma}[Composition]
\label{lem:compose}
$F_\text{MT}$ is, term by term, the sum of $K$ SCROLL objectives (Appendix~\ref{app:scroll}): task $k$'s shared-cavity loss
at $(\mu_k,\Sigma_k,\alpha_k,\eta_k)$, coupled only through the
shared $\psi$---each plate scored against its task's full belief, the
\emph{shared cavity} of the acronym; $K{=}1$ recovers SCROLL.
\end{lemma}

\begin{corollary}[Score optimality of the composed data terms]
\label{cor:mt_score}
Fix the backbone $\psi$ in Definition~\ref{def:fmt}. Then
\textbf{(i)}~each data term of \eqref{eq:fmt} is an exact log score
$-\log m_{k,n}(y^{(k)}_n)$ of the plate predictive $m_{k,n}=\int
p_k(\cdot\mid w_k^\top\psi_n;\eta_k)\,q_k(w_k)\,dw_k$. The log score
is strictly proper. Hence task $k$'s population data risk is
minimised over $(\mu_k,\Sigma_k,\eta_k)$ at the true conditional
$p_\text{true}(y^{(k)}\mid x)$ whenever the predictive family can
represent it; the minimal risk is its conditional entropy.
\textbf{(ii)}~For a Gaussian task,
$p_k=\mathcal{N}(w_k^\top\psi_n,\sigma_\text{obs}^2)$ (so
$\eta_k=\sigma_\text{obs}$), the unrestricted pointwise optimum of the
population data risk is the conditional residual variance
$V^\star_k(x)=\mathbb{E}[(y^{(k)}-\mu_k^\top\psi(x))^2\mid x]$; the
total predictive variance attains its best log-score projection within
$\mathcal{V}_{\psi,k}=\{\sigma_\text{obs}^2+\psi^\top\Sigma_k\,\psi\}$---all
of $V^\star_k$ whenever $V^\star_k\in\mathcal{V}_{\psi,k}$.
\end{corollary}
The claim concerns the data terms; the prior terms---one
observation-free term per task---perturb it by a share that vanishes
as $N$ grows along parameter sequences on which they stay $o(N)$.

\begin{corollary}[Implicit task weighting]
\label{cor:weights}
The data terms of \eqref{eq:fmt} are log scores in nats
(Corollary~\ref{cor:mt_score}), so the tasks enter $F_\text{MT}$ at
unit coefficient per declared task, in one predictive-score currency,
with no unit-correction weights to set; $(\eta_k,\alpha_k)$ are
ordinary arguments of $F_\text{MT}$, learned in the same pass as the
beliefs and the network, and the squared-error coefficient
$(2V_{k,n})^{-1}$ ($\tfrac12\sigma_\text{obs}^{-2}$ at
$\Sigma_k{=}0$) occupies the MAP weight's slot ($\lambda_k$).
\end{corollary}

Both corollaries hold at fixed $\psi$; joint training need not
preserve per-task optimality---the interference channel
(Appendix~\ref{app:future}). In a Gaussian task the likelihood
parameter is the noise scale, $\eta_k=\sigma_\text{obs}$.

\begin{proposition}[Scale equivariance]
\label{prop:scale}
Rescale a Gaussian task's targets, $y^{(k)}\!\mapsto\!c\,y^{(k)}$ with
$c>0$, and map $(\mu_k,\Sigma_k,\sigma_\text{obs}^2,\alpha_k)\mapsto
(c\,\mu_k,\,c^2\Sigma_k,\,c^2\sigma_\text{obs}^2,\,\alpha_k/c^2)$. Then
$F_\text{MT}$ of Definition~\ref{def:fmt} changes by an additive
constant in $c$ alone, so
the map is a bijection between the two problems, leaving the minimiser
over $\psi$ unchanged. A fixed weighted sum of raw losses admits no
such \emph{parameter-only} map: its task-$k$ term scales by $c^2$, and
recovering the same minimiser requires retuning the weights,
$\lambda_j\!\mapsto\!c^2\lambda_j$ for every other task.
\end{proposition}

\begin{corollary}[When weighting is inert]
\label{cor:inert}
If a single $\psi^\star$ minimises every per-task population objective at
once, it minimises every positively weighted sum of them: no weighting
improves on another.
\end{corollary}

Neither statement makes weights dispensable in general. When no one
$\psi$ serves all tasks, the weights select a point on a trade-off
surface that no scale-free principle picks---composition fixes the
units, not the preference. Proofs of all statements are in
Appendix~\ref{app:proofs}.

\section{Experiments}
\label{sec:experiments}


\paragraph{Setup and systems.}
Given the Markov state $x_t$ and a fixed lead time $\Delta$, we forecast
a regression, a binary, and an ordinal observable of the window
$[t,t+\Delta]$: the estimand is the conditional law of the window's
observables---for the state head, the lead-$\Delta$ transition
kernel---and fixing $\Delta$ makes score-optimality testable:
$V^\star(x)$ \emph{is} that kernel's conditional variance.
\textbf{Ornstein--Uhlenbeck (OU)}: $V^\star(\Delta)$
analytic and state-independent---the homoscedastic anchor with known
exact kernel. \textbf{Stochastic Lorenz-63}
\citep{lorenz1963deterministic}: heteroscedastic (cond.\ sd
$p10/p90 = 1.5/8.4$ at $\Delta=1$), ground truth by Monte-Carlo
re-simulation---the test of input-dependent variance.
\textbf{Beijing PM$_{2.5}$} \citep{liang2015pm25}:
the real-data arm, no computable $V^\star$, with the likelihood
choices dictated by regulation. Splits are chronological on the real
series, random and non-overlapping on the SDE corpora.

\paragraph{Methods.}
\emph{SCROLL-MT} minimises the objective~\eqref{eq:fmt} with Gaussian,
thresholded-Gaussian (probit), and ordinal-probit heads
jointly---no unit correction, no outer weight search---under three
belief-covariance models: None ($\Sigma_k{=}0$, only the learned
scale $\sigma_\text{obs}$), Diag (diagonal $\Sigma_k$), and Full
($\Sigma_k$ unrestricted); Diag is the default we report as
\emph{SCROLL-MT} (family choice: Appendix~\ref{app:results}).
\emph{MAP} baselines minimise
$\mathrm{MSE} + \lambda_1 L_1 + \lambda_2 L_2$ with cross-entropy or
ordinal-probit heads, at fixed $\lambda_1{=}\lambda_2{=}1$ and
grid-tuned ($49$ runs, validation-selected); \emph{Kendall} arms learn
the weights instead \citep{kendall2018multi}---the point-estimate
counterpart of SCROLL-MT; \emph{MLE $\sigma(x)$} arms
\citep{nix1994mean} add a per-input regression sd head and train all
losses as NLLs: input-dependent variance, no belief. All methods
share the same candidate architectures and per-fit training
budget---a single hidden layer of width $H{=}50$
with per-task linear heads, heavily overparameterised for a
$1$--$3$-dimensional state, which is what makes the OU anchor below a
real test---with backbones $\{$relu, relu+LN, tanh, tanh+LN$\}$
validation-selected per dataset and seed, $10$ seeds. Setup, splits,
grids: Appendix~\ref{app:details}; predictive scope and remaining
arms (Laplace, ensembles): Appendices~\ref{app:results}
and~\ref{app:future}.

\begin{figure}[t]
\centering
\begin{tikzpicture}
\begin{groupplot}[
  group style={group size=2 by 1, horizontal sep=1.35cm},
  width=0.52\linewidth, height=3.0cm,
  tick label style={font=\small}, label style={font=\small},
  legend style={font=\small, draw=none, fill=none},
]
\nextgroupplot[title={\small OU (analytic $V^\star$)},
  xlabel={lead time $\Delta$},
  legend pos=south east, xmode=log, log ticks with fixed point,
  xtick={0.1,0.25,0.5,1,2}]
\addplot[black, thick] table[x=delta, y=sd_star]
  {figures/data/sde_sigma_ou.dat};
\addlegendentry{$\sqrt{V^\star(\Delta)}$}
\addplot[blue, only marks, mark=*, mark size=1.8pt,
  error bars/.cd, y dir=both, y explicit]
  table[x=delta, y=pred_sd_mean,
        y error plus expr=\thisrow{pred_sd_hi}-\thisrow{pred_sd_mean},
        y error minus expr=\thisrow{pred_sd_mean}-\thisrow{pred_sd_lo}]
  {figures/data/sde_sigma_ou.dat};
\addlegendentry{predictive sd}
\nextgroupplot[title={\small Lorenz-63, per state ($\Delta=0.5$)},
  xlabel={MC $V^\star(x_n)$}, ylabel={$V_n$},
  ylabel style={yshift=-0.25cm},
  xmode=log, ymode=log,
  legend style={at={(0.02,0.98)}, anchor=north west, font=\small,
                draw=none, fill=white, fill opacity=0.7, text opacity=1,
                inner sep=1pt, row sep=-1pt},
  legend image post style={mark size=1.6pt},
  legend cell align=left,
  xtick={1,10,100}, ytick={1,10,100}]
\addplot[gray, thick, domain=0.65:185, samples=2, forget plot] {x};
\addplot[blue, only marks, mark=*, mark size=0.5pt, opacity=0.6]
  table[x=vstar, y=v_diag] {figures/data/perstate_lorenz.dat};
\addlegendentry{free (Diag)}
\addplot[teal, only marks, mark=triangle*, mark size=0.7pt, opacity=0.6]
  table[x=vstar, y=v_full] {figures/data/perstate_lorenz.dat};
\addlegendentry{free (Full)}
\addplot[red, only marks, mark=*, mark size=0.5pt, opacity=0.6]
  table[x=vstar, y=v_closed] {figures/data/perstate_lorenz.dat};
\addlegendentry{closed}
\end{groupplot}
\end{tikzpicture}
\caption{Left: total predictive sd of the state forecast vs the analytic
$V^\star(\Delta)$ on OU ($\pm1$ sd over 10 seeds)---within $2\%$ at every
lead. Right: the per-state test on Lorenz-63, predicted $V_n$ against
Monte-Carlo $V^\star(x_n)$ over test states (grey: identity), on the
tanh$+$LN backbone. Both free families track the state-dependent spread;
the closed route's leverage variance is flat across a $284\times$ range of
$V^\star$, on every backbone (Appendix~\ref{app:details}).}
\label{fig:sigma_delta}
\end{figure}
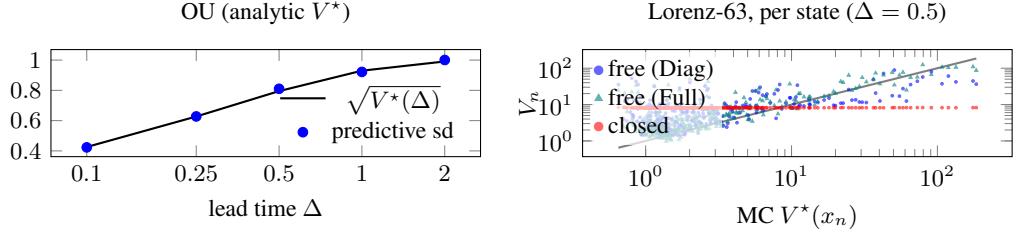

\paragraph{Testbed properties.}
Three corpus properties frame what follows, none a claim about the
method.

\emph{(i) A failure can live in the head's structure, where no weight
reaches it.} The intro's CE-versus-threshold exhibit is that failure,
on the OU barrier event: a controlled misspecification (one CE
bias would repair it).

\emph{(ii) The weight has at most three jobs, and the corpora
separate them}: inert where one $\psi^\star$ serves every task
(Corollary~\ref{cor:inert}; OU, PM$_{2.5}$), carrier of the units,
which composition derives instead (Proposition~\ref{prop:scale}), or
a preference under conflict that no scale-free principle supplies and
we do not claim (Lorenz). Either way the weighting axis is not where
the difference lies: SCROLL-MT matches or beats the better fair grid
per corpus at a fraction of its compute; all three properties are
measured in Appendix~\ref{app:results}.

\emph{(iii) The event head discriminates only on OU}, by observation
model alone; elsewhere every arm's event NLL sits in a $0.05$-nat
band (Table~\ref{tab:main_sde}). Claims below rest on the state
and regime heads.

\paragraph{The well-specified anchor.}
On OU theory fixes the answer---the population optimum is the exact
kernel and \emph{expected} test NLL is floored
(Corollary~\ref{cor:mt_score}(i))---and every arm attains
that floor. The anchor tests the one arm free to miss it:
nothing in \eqref{eq:fmt} caps the $50$-dimensional belief term
$\psi^\top\Sigma\,\psi$, yet it carries no spurious spread
(Figure~\ref{fig:sigma_delta}, left; Appendix~\ref{app:results}), the
run settling at a large $\alpha$ and a contracted belief term.

\paragraph{The predictive variance tracks the state-dependent
spread---and only when freely routed.}
State by state, the predicted $V_n$ correlates with the Monte-Carlo
$V^\star(x_n)$ at Pearson $r{=}0.61$ (Diag) and $0.85$ (Full), rank
$0.56$ and $0.66$ (Figure~\ref{fig:sigma_delta}, right, plots
Diag). The apparent overshoot is
Corollary~\ref{cor:mt_score}(ii) as specified, not error: $V^\star$
omits mean-model error, and against the residual variance about the
model's own mean---what the log score targets---Full sits within $3\%$.
Routing unpins it: at the evidence optimum $V_n$ varies under
$1.3\times$ across a $284\times$ range of $V^\star$---and the binding
map alone, estimand kept, is already this flat
(Appendix~\ref{app:results}). What the freed
family represents decides the payoff---Full beats the pinned route on
every backbone ($-0.14$ to $-0.31$ nats), Diag on the normalised ones
(Appendix~\ref{app:results}); on OU every route ties.
The ladder of Table~\ref{tab:main_sde}
isolates the belief's role: the NLL-trained global scale (None) falls
back to MAP-level quality, the belief restores the headline, and
coupling across the basis (Full) improves it further. Nor is the belief
merely an input-dependent variance head: the MLE $\sigma(x)$ arms match
its calibration but not its NLL, and the probit variant collapses on
the real arm, where the prior-regularised belief never does. Nor is it enough to \emph{have}
the belief: set post hoc by curvature at the MAP point (Laplace), the
same families concede on Lorenz and tie only on
OU---having a belief is not what carries the gain
(Appendix~\ref{app:results}). Nor is the shared backbone free by
assumption. Against each head trained \emph{alone} (only the task
inventory differs), sharing is free on
OU, costs on Lorenz (state and regime
heads), and costs nothing on the real series, where over five
rolling origins the point estimate favours the joint model on all
three heads without resolving it (Appendix~\ref{app:results}). We
claim only what that settles: one feature map serves all three tasks
free where they do not compete, and is paid for where they do.

\begin{table}[t]
\centering\footnotesize
\setlength{\tabcolsep}{4.5pt}\renewcommand{\arraystretch}{0.92}
\caption{Test NLL$\downarrow$ of the three tasks on the three
systems (10 seeds, backbone validation-selected per seed;
OU/Lorenz at $\Delta=0.5$, PM$_{2.5}$ at $24$\,h; reg.\ = state
forecast, evt = event task, ord = ordinal task).
\textbf{Bold} = best per column; \textit{italic} = not
significantly worse (one-sided paired $t$-test, $p\ge0.05$);
``Tied'' in the text means this, not equivalence.
Ens-$5$ = five-member deep ensemble of the arm named
(construction and caveats in Appendix~\ref{app:results});
every other row is a single run.
The CE variant of the MLE $\sigma(x)$ arm is omitted here for
space; it carries no claim the probit variant does not.
Calibration, accuracy, wall-clock, and that row:
Table~\ref{tab:full_sde}.}
\label{tab:main_sde}
\begin{tabular}{lrrrrrrrrr}
\toprule
 & \multicolumn{3}{c}{OU} & \multicolumn{3}{c}{Lorenz-63} & \multicolumn{3}{c}{PM$_{2.5}$ (real)} \\
\cmidrule(lr){2-4} \cmidrule(lr){5-7} \cmidrule(lr){8-10}
 & reg. & evt & ord & reg. & evt & ord & reg. & evt & ord \\
\midrule
SCROLL-MT (Diag) & 1.204 & 0.271 & 0.881 & 2.333 & 0.170 & 1.032 & 1.250 & 0.450 & 1.203 \\
SCROLL-MT (Full) & 1.205 & 0.265 & 0.881 & \textbf{2.230} & 0.169 & \textbf{0.843} & \textit{1.255} & 0.489 & 1.212 \\
SCROLL-MT (None, no belief) & 1.204 & 0.271 & \textit{0.880} & 2.596 & 0.168 & 1.069 & 1.311 & 0.451 & 1.291 \\
\midrule
MAP, CE heads, $\lambda{=}1$ & \textit{1.203} & 0.417 & 0.954 & 2.613 & 0.174 & 1.154 & 1.312 & 0.451 & 1.229 \\
MAP, CE heads, $\lambda$ grid & 1.203 & 0.417 & 0.954 & 2.621 & \textbf{0.166} & 0.907 & 1.318 & 0.456 & 1.229 \\
Kendall, CE heads & \textbf{1.202} & 0.417 & 0.954 & 2.798 & 0.180 & 1.175 & 1.312 & 0.458 & 1.238 \\
MAP, probit heads, $\lambda{=}1$ & \textit{1.203} & 0.270 & \textit{0.880} & 2.613 & 0.175 & 1.252 & 1.312 & 0.452 & 1.288 \\
MAP, probit heads, $\lambda$ grid & \textit{1.203} & \textbf{0.263} & \textbf{0.879} & 2.620 & \textit{0.169} & 1.075 & 1.334 & 0.460 & 1.297 \\
Kendall, probit heads & 1.205 & 0.266 & 0.884 & 2.810 & 0.180 & 1.249 & 1.312 & \textit{0.440} & 1.291 \\
MLE $\sigma(x)$, probit heads & \textit{1.203} & 0.269 & \textit{0.880} & 2.612 & 0.189 & 1.245 & \textit{3.011} & 0.458 & 1.296 \\
Laplace (Diag), probit heads & \textit{1.203} & 0.270 & \textit{0.880} & 2.613 & 0.175 & 1.253 & 1.306 & 0.450 & 1.287 \\
SCROLL-MT (Diag), Ens-$5$ & 1.204 & 0.271 & \textit{0.881} & 2.295 & \textit{0.167} & 1.004 & \textbf{1.243} & \textit{0.443} & \textbf{1.197} \\
MAP, CE heads, Ens-$5$ & \textit{1.202} & 0.417 & 0.954 & 2.608 & 0.172 & 1.138 & 1.303 & \textbf{0.440} & 1.215 \\
MAP, probit heads, Ens-$5$ & \textit{1.203} & 0.270 & \textit{0.880} & 2.607 & 0.173 & 1.244 & 1.308 & \textit{0.445} & 1.279 \\
\bottomrule
\end{tabular}
\end{table}

\paragraph{The prior it fits along the way needs no loop.}
Freezing the three precisions at a shared $\alpha$ and tuning that
explicitly ($7$ values) is the
counterfactual to Corollary~\ref{cor:weights}'s single-pass
clause. The grid is \emph{not} flat---on Lorenz a strong prior doubles
calibration error and costs $0.12$ nats on the ordinal head---yet the
self-tuned run sits within $0.006$ nats of the grid-selected
regression NLL on every corpus, within $0.017$ on every head
(Table~\ref{tab:alpha_sde}), from one prior configuration per
backbone rather than seven, and lands
there across regimes three decades apart (median fitted
$\alpha$: $379$ on OU, $0.11$ on Lorenz; scale convention in
Appendix~\ref{app:details}). Whether the
tasks want \emph{different} $\alpha_k$ is open (Appendix~\ref{app:future}).

\section{Conclusion and outlook}
\label{sec:outlook}


This is a deliberately scoped first study: two controlled
low-dimensional SDE systems, one small real series, $K{=}3$ tasks,
one lead time per run. The scope is what buys ground truth: on the
SDE corpora $V^\star(x)$ is computable, so the trained predictive is
checked against the exact one. Within it, composed likelihoods
did their job. On the real data they are best or statistically tied
in every column among single-run methods (its Ens-$5$ among all), and the
state margin survives five rolling origins. They are best where the
belief has work to do (Lorenz), and within $0.01$ nats of the best
tuned grid on OU (Table~\ref{tab:main_sde}). All without an outer
weight search, at one configuration per backbone: the gains come
from the observation model and the freely routed belief, not the
weighting axis. Theory specified that
outcome in advance, and mostly SCROLL's own. The composition
adds no new principle: the factor graph as written already composes
(Lemma~\ref{lem:compose}). The heteroscedastic win is single-task
score optimality carried through
(Corollary~\ref{cor:mt_score}). New is the belief inside the
currency: weight slots carry score-trained predictive
variances, not learned scalars (Corollary~\ref{cor:weights}). Open
directions are in Appendix~\ref{app:future}.

\clearpage
\bibliographystyle{plainnat}
\bibliography{refs}

\clearpage
\appendix
\section{SCROLL in brief}
\label{app:scroll}


A compact summary of the results of \citet{prochazka2026bethe} that this
paper uses: enough to follow every claim made here, with the derivations
and the wider design space they sit in left to the source.

\paragraph{Objective.}
SCROLL derives its loss from the Bethe free energy of the model's
factor graph by keeping each factor's \emph{local normaliser}: one
log-partition term per factor, scoring the overlap between the local
belief and that factor. For a model with a deterministic backbone
$\psi$, a likelihood $p(y\mid w^\top\psi;\eta)$ with parameters
$\eta$ (a task with Gaussian likelihood is the body's \emph{Gaussian
task}), a Gaussian last-layer belief $q(w)=\mathcal{N}(\mu,\Sigma)$, a
prior $p(w)=\mathcal{N}(0,\alpha^{-1}I)$, and $N$ observation plates,
every intermediate (deterministic) factor's normaliser is unity,
$-\log Z_a=0$ (its log-partition reduces to a feasibility
constraint), leaving, for a single task, the \emph{shared-cavity
loss}
\begin{equation}
  F_\text{SC}
  \;=\; -\sum_{n=1}^{N} \log \underbrace{\int
      p\bigl(y_n \mid w^\top\psi_n;\eta\bigr)\, q(w)\, dw}_{Z_n}
  \;-\; \log \underbrace{\int p(w)\, q(w)\, dw}_{Z_w}
  \label{eq:bethe}
\end{equation}
($K{=}1$ of \eqref{eq:fmt}): $N$ data terms
and one prior term. One scope note: the source paper defines
$F_\text{SC}$ on the whole factor graph, with no last-layer
restriction; \eqref{eq:bethe} is that definition specialised to this
model class, where the deterministic backbone factors contribute
zero. Under the \emph{shared cavity} each plate is scored
against the full belief $q$---the reaction-free limit that decouples the
loss into a batchable per-plate sum---and backbone parameters train by
ordinary backpropagation.

$F_\text{SC}$ is not the Bethe free energy itself. Evaluated at the
belief's configuration---tilted factor beliefs $b_a\propto f_a\,q$,
variable belief $q$, which is generally \emph{not} locally
consistent---the free energy decomposes as
$F_\text{Bethe}\bigl(b(q)\bigr)=F_\text{SC}(q)+C_w(q)$, with remainder
$C_w(q)=\sum_a\bigl(\mathbb{E}_{b_a}-\mathbb{E}_q\bigr)[\log q]-H(q)$
collecting the belief-mismatch and counting terms, generically
non-zero (the source paper's Bethe decomposition of the shared-cavity
loss and the closed form of its remainder). The free
energy supplies the terms; its consistency constraints are not
imposed---provenance, not enforcement. What is minimised is
$F_\text{SC}$: a composite predictive-score objective---the marginal
log-score sum plus the prior-overlap term---and an estimand in its own
right (\emph{predictive consistency} below), not an approximation of
the evidence $-\log Z$. The same reading scopes the composed
objective: the data terms of \eqref{eq:fmt} are a composite sum of
marginal predictive
log scores---proper for each task's marginal, it represents no
cross-task residual dependence and no joint predictive over the
observables (Appendix~\ref{app:future} studies that extension).

\paragraph{Data and prior terms.}
Each data term is the negative log of the plate's predictive
normaliser, $-\log Z_n=-\log m_n(y_n)$ with
$m_n(y)=\int p(y\mid w^\top\psi_n;\eta)\,q(w)\,dw$. For a Gaussian
task ($\eta=\sigma_\text{obs}$) this is the predictive NLL of $y_n$ under the plate variance
$V_n=\sigma_\text{obs}^2+\psi_n^\top\Sigma\,\psi_n$. The discrete
heads are probits on the same latent line, the belief convolution
widening the probit scale to
$D_n=\sqrt{c^2+\psi_n^\top\Sigma\,\psi_n}$: with cutpoints
$\tau_1<\dots<\tau_{R-1}$ and sentinels $\tau_0=-\infty$,
$\tau_R=\infty$, the ordinal head's bucket mass is
\begin{equation*}
  P\bigl(y_n=r\mid x_n\bigr)
  = \Phi\!\Bigl(\tfrac{\tau_r-\mu^\top\psi_n}{D_n}\Bigr)
  - \Phi\!\Bigl(\tfrac{\tau_{r-1}-\mu^\top\psi_n}{D_n}\Bigr),
\end{equation*}
and the thresholded-Gaussian event head is the $R{=}2$ case,
$P(y_n{=}1\mid x_n)=\Phi\bigl((\mu^\top\psi_n-\tau_1)/D_n\bigr)$. For
identification the probit scale is fixed ($c{=}1.005$ in the shipped
implementation) and the cutpoints are learned, parametrised strictly
increasing as $\tau_r=\tau_1+\sum_{j<r}e^{\delta_j}$: the physical
thresholds define the \emph{labels} only, and the learned cutpoint is
what lets the head absorb an offset a bias-free CE head cannot
(Section~\ref{sec:experiments}). All terms are one-dimensional
Gaussian integrals in closed form. The prior term is the Gaussian
overlap $-\log\mathcal{N}(\mu\mid0,\Sigma+\alpha^{-1}I)$ (up to a
constant).

\paragraph{Routing: who selects the belief.}
The \emph{binding map} sends the remaining parameters---backbone
$\psi$, prior precision $\alpha$, likelihood parameters---to the belief
minimising
$\mathrm{KL}\bigl(\mathcal{N}(\mu,\Sigma)\,\|\,p(w\mid\mathcal{D})\bigr)$,
where $p(w\mid\mathcal{D})$ is the exact posterior given the cavity
data (under the shared cavity, all $N$ plates) at those parameter
values: the projection of that posterior onto the Gaussian
family---closed-form under conjugacy, where it coincides with the
local-consistency fixed point. Under the Gaussian head that closed
form is
$\Sigma_\text{post}=(\alpha I+\sigma_\text{obs}^{-2}\Psi^\top\Psi)^{-1}$,
$\mu_\text{post}=\sigma_\text{obs}^{-2}\,\Sigma_\text{post}\Psi^\top y$:
its variance depends on the design $\Psi$ alone, never on the
residuals---the residual-independent leverage form of the pinned
route. \emph{Closed} routing imposes this map as a constraint;
\emph{free} routing drops it and makes $(\mu,\Sigma)$ direct
optimisation variables, the objective itself selecting the belief. Put
plainly: the closed route takes the belief the posterior prescribes,
the free route lets the objective choose it.
SCROLL and every SCROLL-MT arm in this paper use the free route; the
Laplace arm is a posterior-bound comparator, while the remaining
point-estimate baselines carry no belief. Imposing the map while keeping the data
terms of \eqref{eq:bethe} constrains the belief coordinates of one
objective and nothing else; replacing a head's terms by the conjugate
model's own evidence is a different cell again, changing the estimand
from a composite predictive score to a marginal likelihood. The two are
worth separating, and the source paper separates them.

One consequence of free routing deserves stating outright, because the
notation invites the other reading: $\psi^\top\Sigma\,\psi$ is a
score-trained conditional spread, not a posterior epistemic variance. It
is fitted to predict residuals, so it need not contract as $N$ grows,
and it carries no claim about confidence in $w$.

\paragraph{Predictive consistency (proper score).}
For \emph{any} observation factor, $-\log m_n(y_n)$ is the log score of
the predictive $m_n$. The log score is strictly proper, so the
population data term decomposes into the expected conditional entropy of
the true conditional plus
$\mathbb{E}_x\,\mathrm{KL}\bigl(p_\text{true}(\cdot\mid x)\,\|\,
m(\cdot\mid x)\bigr)$, and is minimised---over beliefs whose predictive
can represent the truth---exactly at $m=p_\text{true}$. No conjugacy and
no tree structure is required. This is the result behind the
entropy-floor test on OU. Two qualifications, both carried by the one
prior term. First, the claim concerns the data terms: the prior term
contains no observation---one term against $N$---so, along parameter
sequences on which it remains $o(N)$, minimising \eqref{eq:bethe} is
score-optimal up to a prior perturbation whose per-plate share
vanishes---the sense of Corollary~\ref{cor:mt_score}. (The
remainder $C_w$ above plays no role here: it separates $F_\text{SC}$
from the free energy, not from the score.) Second, it is a population
statement, and at finite $N$ the data terms see
$(\Sigma,\sigma_\text{obs})$ only through the totals $V_n$, so what
selects the \emph{split} between $\sigma_\text{obs}^2$ and
$\psi^\top\Sigma\,\psi$ is that same single prior term, whose
normalised contribution vanishes along the $o(N)$ parameter sequences
above. The total is identified long before the split is.
Composition changes neither qualification in kind: the data terms of
\eqref{eq:fmt} are unchanged individual log scores, so their propriety
is composition-invariant by construction, and only the count of
observation-free terms---here one prior term per task, $K$ against
$KN$ plates, hence the same per-plate rate under the same $o(N)$
condition---depends on how the objectives are composed.

\paragraph{The score optimum and the pinned variance.}
With the mean function fixed, the regression data term sees
$(\Sigma,\sigma_\text{obs})$ only through the plate variances $V_n$,
which are minimised in population at the conditional residual variance
$V^\star(x)=\mathbb{E}[(y-\mu^\top\psi(x))^2\mid x]$ of
Section~\ref{sec:background}---the pointwise unrestricted optimum. The
binding map instead pins $V_n$ to a residual-independent leverage
form, which realises $V^\star$ only when residuals are
homoscedastic---the pinning is the map's doing, whatever the estimand
(Appendix~\ref{app:results} separates the two), and with the evidence
estimand the pinned cell is the corner of the source design space. What
separates the free route from the pinned one splits in two, against
the family the last layer can express,
$\mathcal{V}_\psi=\{\sigma_\text{obs}^2+\psi^\top\Sigma\,\psi\}$: a
\emph{routing} gap, between the pinned leverage profile
(asymptotically constant under balanced leverage) and the best member of
$\mathcal{V}_\psi$, which free routing removes and only that; and a
\emph{representation} gap, between that best member and $V^\star$
itself, which no routing closes, though training $\psi$ can shrink it.
What free routing recovers is therefore the \emph{expressible} residual
heteroscedasticity, the whole of it only when
$V^\star\in\mathcal{V}_\psi$. Both gaps are measured in
Appendix~\ref{app:results}: a diagonal belief on an unnormalised
backbone cannot express $V^\star$ and gains nothing from being freed,
while a full one expresses enough to gain on every backbone.
Mechanically, the free route drives $\sigma_\text{obs}^2$ toward the
input-independent floor and lets $\psi^\top\Sigma\,\psi$ carry the
input-dependent remainder---the split verified against
simulator ground truth.

\section{Proofs of the formal statements}
\label{app:proofs}

\begin{proof}[Proof of Lemma~\ref{lem:compose}]
Group \eqref{eq:fmt} by $k$. Task $k$'s block is, term by term,
\eqref{eq:bethe} evaluated at $(\mu_k,\Sigma_k,\alpha_k,\eta_k)$: its
$N$ data terms are the plate normalisers $-\log Z_{k,n}$,
$Z_{k,n}=\int p_k\bigl(y^{(k)}_n\mid
w_k^\top\psi_n;\eta_k\bigr)\,q_k(w_k)\,dw_k$, and its prior term the
overlap $-\log\int p_{\alpha_k}(w_k)\,q_k(w_k)\,dw_k$---exactly the
single-task SCROLL objective for task $k$'s head, plates, and prior.
The blocks share no parameter beyond $\psi$; $K{=}1$ is immediate.
\end{proof}

\begin{proof}[Proof of Corollary~\ref{cor:mt_score}]
Fix $\psi$. The blocks of Lemma~\ref{lem:compose} then share no
argument, so minimising $F_\text{MT}$ over
$\{(\mu_k,\Sigma_k,\alpha_k,\eta_k)\}_k$ minimises each block
separately, and the population data risk of the sum is the sum of the
per-task risks: the single-task statements apply block by block.
Predictive consistency (above) applied to task $k$'s data terms gives
\textbf{(i)}; the pinned-variance argument (above) applied to task
$k$'s Gaussian plate variances
$V_{k,n}=\sigma_\text{obs}^2+\psi_n^\top\Sigma_k\,\psi_n$ gives
\textbf{(ii)}.
\end{proof}

\begin{proof}[Proof of Corollary~\ref{cor:weights}]
The data terms are log scores by Corollary~\ref{cor:mt_score}(i), in
nats by convention, and $(\eta_k,\alpha_k)$ enter \eqref{eq:fmt} as
ordinary arguments of a single differentiable objective, so the pass
that trains $(\mu_k,\Sigma_k,\psi)$ fits them too. For the slot
identity, write the Gaussian data term at plate $n$ as
$\tfrac{1}{2V_n}\,r_n^2+\tfrac12\log(2\pi V_n)$ with residual $r_n$:
the factor multiplying the squared error is $(2V_n)^{-1}$, so the
total predictive precision $V_n^{-1}$ occupies the slot where the MAP
weight $\lambda_k$ multiplies the squared error, reducing to the
fitted noise precision $\sigma_\text{obs}^{-2}$ when $\Sigma_k{=}0$
($V_n=\sigma_\text{obs}^2$, the beliefless case), with the $\log V_n$
term as the barrier that keeps it finite---the mechanism of
\citet{kendall2018multi} with the belief term added to $V_n$.
\end{proof}

\begin{proof}[Proof of Proposition~\ref{prop:scale}]
Write task $k$'s data term at plate $n$ as
$\tfrac12\log(2\pi V_{k,n})+r_{k,n}^2/(2V_{k,n})$, with residual
$r_{k,n}=y^{(k)}_n-\mu_k^\top\psi_n$ and
$V_{k,n}=\sigma_\text{obs}^2+\psi_n^\top\Sigma_k\,\psi_n$. Under the map
$r_{k,n}\mapsto c\,r_{k,n}$ and $V_{k,n}\mapsto c^2V_{k,n}$, so every
quadratic term is unchanged and each log term gains $\log c$: the $N$
data terms contribute $N\log c$. The prior term is the Gaussian overlap
$-\log\mathcal{N}(\mu_k\mid0,\Sigma_k+\alpha_k^{-1}I)$, and the map sends
$\Sigma_k+\alpha_k^{-1}I\mapsto c^2(\Sigma_k+\alpha_k^{-1}I)$ against
$\mu_k\mapsto c\,\mu_k$, so its quadratic form is likewise unchanged and
its log-determinant gains $H\log c$. Every other task is untouched, so
$F_\text{MT}$ gains $(N+H)\log c$---a constant in $\psi$ and in all
parameters, leaving $\nabla_\psi F_\text{MT}$ and hence the minimiser
unchanged. For a weighted sum
$L_k+\sum_{j\neq k}\lambda_jL_j$, with $L_k$ the homogeneous raw
squared-error loss, the mean's best response to the
rescaling is $\mu_k\mapsto c\,\mu_k$, which leaves the task-$k$ loss at
$c^2L_k$; since $\arg\min_\psi\{c^2L_k+\sum_{j\neq k}\lambda_jL_j\}
=\arg\min_\psi\{L_k+\sum_{j\neq k}(\lambda_j/c^2)L_j\}$, the original
minimiser is recovered only by $\lambda_j\mapsto c^2\lambda_j$.
Corollary~\ref{cor:inert} is immediate: if $\psi^\star$ minimises each
$L_k$ then it minimises $\sum_k\lambda_kL_k$ for any $\lambda_k>0$.
\end{proof}

Proposition~\ref{prop:scale} concerns the ideal covariance families;
the implemented $\varepsilon I$ floor on $\Sigma$ lies outside them,
and Appendix~\ref{app:results} measures that it does not bind in the
tested range.

\section{Related work}
\label{app:related}
Loss weighting in MTL appears in the body along its main axes: grid
search, learned uncertainty weighting \citep{kendall2018multi}, and
gradient surgery \citep{chen2018gradnorm,yu2020pcgrad} (orthogonal to
composition, Appendix~\ref{app:future}). \citet{kendall2018multi}
already draw the no-weights conclusion from the likelihood view for
point-estimate MAP training, and their mechanism composes with any
observation model once it is written down---the Kendall-probit arms of
Section~\ref{sec:experiments} do exactly that; what the shared-cavity
objective adds is the belief---hence a calibrated, input-dependent predictive
variance---and a single-pass fit of the priors and likelihood
scales (Section~\ref{sec:experiments}). Composing heterogeneous
likelihoods over a shared latent function is the construction of
heterogeneous multi-output GPs \citep{moreno2018heterogeneous};
SCROLL-MT shares the heterogeneous-likelihood composition motif, with
the latent function replaced by a
trained backbone and the posterior by per-task last-layer beliefs.
NLL-trained variance heads go back to \citet{nix1994mean} and power
deep ensembles \citep{lakshminarayanan2017simple}; the None ablation of
Table~\ref{tab:main_sde} is their per-task-scale analogue inside the
composed objective, and the belief term
$\psi^\top\Sigma_k\,\psi$ is what they lack. Bayesian last layers
(e.g.\ variational, \citealp{harrison2024variational}) are compared
against SCROLL at length in \citet{prochazka2026bethe}; proper-scoring
terminology follows \citet{gneiting2007strictly}. The
sequential-cavity and adaptive-filtering lineage is developed in
Appendix~\ref{app:future}.

\section{Benchmark generation and training details}
\label{app:details}


\paragraph{OU corpus.}
$dX=-\theta X\,dt+\sigma\,dW$ with $\theta=1$, $\sigma=\sqrt2$
(stationary $\mathcal{N}(0,1)$), simulated by the exact one-step
transition
$X_{t+h}=e^{-\theta h}X_t+\sqrt{\tfrac{\sigma^2}{2\theta}\bigl(1-e^{-2\theta h}\bigr)}\,\xi$,
$\xi\sim\mathcal{N}(0,1)$, on a grid $h=0.01$---sampling the analytic
transition density, so there is no discretisation bias. The lead-$\Delta$
conditional variance is
$V^\star(\Delta)=\tfrac{\sigma^2}{2\theta}\bigl(1-e^{-2\theta\Delta}\bigr)$,
state-independent. Each of the
$6{,}500$ windows starts from an independent stationary draw (one window
per realisation; the corpus is i.i.d.\ by construction). The path
functionals are monitored on the grid, both endpoints included---the
event and average labels therefore carry the grid's discretisation,
though the state transition itself is exact: the
event task is $\mathds{1}[\max_{[t,t+\Delta]}X\ge1.5]$ (barrier at
$1.5$ stationary sd), the ordinal task buckets the window time-average
into corpus-level quartiles ($K{=}4$). Generation uses a fixed seed:
the corpus is frozen once, like a UCI table, and the run seed controls
only split and initialisation.

\paragraph{Stochastic Lorenz-63 corpus.}
The $(10,28,\nicefrac83)$ drift with additive noise of amplitude $2.0$
on all three coordinates, integrated by Euler--Maruyama at
$dt=0.005$. $65$ parallel chains are initialised near $(1,1,25)$,
burned in for $2{,}000$ steps ($10$ time units), and then sampled with a
decorrelation gap of $100$ steps ($0.5$ time units) between the end of
one window and the start of the next, giving $6{,}500$ windows in
total. Tasks: $y^{(1)}=x(t{+}\Delta)$, $y^{(2)}=\mathds{1}[x(t{+}\Delta)>0]$
(attractor lobe), $y^{(3)}=$ corpus-level quartile of $z(t{+}\Delta)$.

\paragraph{Ground truth and the per-state probe.}
OU: $V^\star(\Delta)$ analytic (Section~\ref{sec:experiments}). Lorenz:
the conditional sd of $x(t{+}\Delta)$ is estimated by re-simulating
$200$ ensemble members per state with fresh noise. For the averaged
view of Figure~\ref{fig:sigma_delta_lorenz} this is done for $2{,}000$
attractor states and the plotted target is
$\smash{\sqrt{\mathbb{E}[\mathrm{Var}(x_{t+\Delta}\mid x_t)]}}$---the
score-optimal value for a single $\sigma_\text{obs}$, i.e.\ $V^\star$
with the heteroscedasticity gap averaged out
(Appendix~\ref{app:scroll})---with the band spanning the per-state
p10--p90 conditional sd. The per-state panel of
Figure~\ref{fig:sigma_delta} instead re-simulates \emph{every} test
state, so predicted and true variances are paired point by point; $V^\star$
spans $284\times$ across those states. Both are $10$-seed runs of the
protocol above at $\Delta{=}0.5$, with correlations reported as
mean~$\pm$~sd over seeds and the plotted cloud from the first seed,
subsampled by $V^\star$-rank for legibility. The Monte-Carlo targets
are themselves noisy at $200$ members per state, so the reported
correlations and the estimated $V^\star$ range carry simulation
error; the broad contrast between the pinned route's $1.3\times$
range and the estimated $284\times$ target range is nevertheless
large, and we attach no inferential meaning to the precise
correlation or range values.

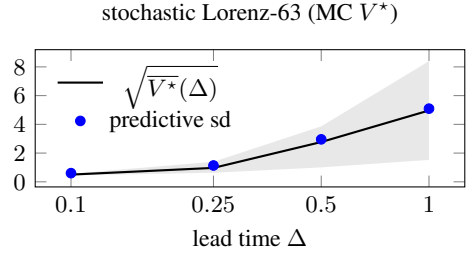
\begin{figure}[t]
\centering
\begin{tikzpicture}
\begin{axis}[
  width=0.52\linewidth, height=3.4cm,
  xlabel={lead time $\Delta$}, title={\small stochastic Lorenz-63 (MC $V^\star$)},
  tick label style={font=\small}, label style={font=\small},
  legend style={font=\small, draw=none, fill=none},
  legend pos=north west, xmode=log, log ticks with fixed point,
  xtick={0.1,0.25,0.5,1}]
\addplot[name path=lo, draw=none, forget plot] table[x=delta, y=sd_p10]
  {figures/data/sde_sigma_lorenz.dat};
\addplot[name path=hi, draw=none, forget plot] table[x=delta, y=sd_p90]
  {figures/data/sde_sigma_lorenz.dat};
\addplot[gray, opacity=0.18, forget plot] fill between[of=lo and hi];
\addplot[black, thick] table[x=delta, y=sd_star]
  {figures/data/sde_sigma_lorenz.dat};
\addlegendentry{$\sqrt{\overline{V^\star}(\Delta)}$}
\addplot[blue, only marks, mark=*, mark size=1.8pt,
  error bars/.cd, y dir=both, y explicit]
  table[x=delta, y=pred_sd_mean,
        y error plus expr=\thisrow{pred_sd_hi}-\thisrow{pred_sd_mean},
        y error minus expr=\thisrow{pred_sd_mean}-\thisrow{pred_sd_lo}]
  {figures/data/sde_sigma_lorenz.dat};
\addlegendentry{predictive sd}
\end{axis}
\end{tikzpicture}
\caption{The averaged view of the tracking result on Lorenz-63 ($\pm1$ sd over $10$
seeds; grey band: per-state conditional sd $p10$--$p90$). The predictive sd
targets the $V^\star$ average (black) at every lead, overshooting at short
leads where mean-model error enters. Averaging is exactly what the per-state
panel of Figure~\ref{fig:sigma_delta} removes: a homoscedastic head can match
this curve, and the no-belief arm nearly does.}
\label{fig:sigma_delta_lorenz}
\end{figure}

\paragraph{Beijing PM$_{2.5}$.}
Hourly PM$_{2.5}$ and meteorology, 2010--2014 \citep{liang2015pm25}.
Gaps of up to $6$\,h in the PM$_{2.5}$ series are linearly interpolated;
windows touching longer gaps are dropped. Windows are non-overlapping
(stride $=$ lead $=24$\,h), leaving $N{=}1{,}680$. Input
($D{=}36$): the last $24$ hourly $\log(1{+}\text{PM}_{2.5})$ values,
current dew point, temperature, pressure, cumulated wind speed, snow
and rain hours, a $4$-way wind-direction one-hot, and month
$\sin$/$\cos$. Targets: $y^{(1)}=\log(1{+}\text{PM}_{2.5})$ at
$t{+}24$\,h; $y^{(2)}=\mathds{1}[\text{mean PM}_{2.5}\text{ over }
(t,t{+}24\,\text{h}]>75\,\mu\text{g}/\text{m}^3]$;
$y^{(3)}=$ band of PM$_{2.5}$ at $t{+}24$\,h with edges
$35/75/150\,\mu\text{g}/\text{m}^3$. The headline split is
chronological $60/20/20$, on which the run seed varies initialisation
only; the rolling origins (Appendix~\ref{app:results}) vary the
split itself.

\paragraph{Training protocol (all methods).}
Inputs standardised and regression targets centred on training-set
statistics; Adam with learning rate $0.03$, at most $5{,}000$
full-batch steps, early stopping on the summed validation NLL. Every
objective is optimised together with a backbone $\ell_2$ term
$\tfrac{\lambda_\psi}{2}\|W\|^2$ on the single hidden-layer weight
matrix, $\lambda_\psi{=}0.01$; the backbone carries no bias and, in
the LN arms, LayerNorm is non-affine (no learned gain or bias), so
every parameter that can globally rescale the features is regularised.
These conventions close the feature-scale gauge (next paragraph):
fitted $\alpha_k$ are interpretable only
relative to them, while predictive means and total variances are
gauge-invariant. The MAP
$\lambda$ grid is $\{0.01,0.1,0.3,1,3,10,30\}^2$ ($49$ runs,
validation-selected). Kendall arms train per-task log-scales $s_k$
jointly with the network \citep{kendall2018multi}: regression loss
$\tfrac12 e^{-2s}\,\mathrm{MSE}+s$ (so $e^{s}$ \emph{is} the reported
predictive sd) and $e^{-2s}\,\mathrm{NLL}+s$ per classification head.
The backbone variant $\{$relu, relu+LN, tanh, tanh+LN$\}$ is selected
per dataset, seed, and method by validation loss, and every arm of
Table~\ref{tab:main_sde} runs the identical protocol. ``One run''
throughout therefore means one weight/hyperparameter configuration;
the four-backbone validation sweep is common to every arm and is
included in every cost figure. Regression
calibration error is the mean absolute coverage gap of the central
predictive intervals,
$\operatorname{mean}_a\bigl|\widehat{\Pr}\bigl[y \in \mu(x) \pm
z_{(1+a)/2}\,\sigma(x)\bigr]-a\bigr|$ over nominal levels
$a\in\{0.05,0.10,\dots,0.95\}$ on the test set.

\paragraph{Learning the hyperparameters in one pass.}
The prior precision $\alpha$ and the likelihood parameters
($\sigma_\text{obs}$, ordinal cutpoints) enter the objective as ordinary
arguments, so the same gradient pass that trains $(\mu,\Sigma)$ and the
backbone also estimates them, with no outer tuning loop. They are
fitted to $F_\text{MT}$ itself---a predictive-score objective plus
prior overlap under the
shared cavity, not a marginal likelihood. An absolute $\alpha$ means
something only against a scale convention, since the backbone
parameterisation and its $\ell_2$ term fix the scale of $\psi$ and
$\alpha$ trades against that scale through the prior-overlap term:
every $\alpha$ quoted in this paper is under the single fixed
convention of the training protocol above, so
the cross-corpus comparison is a ratio within it rather than an
absolute magnitude \citep{prochazka2026bethe}. The gauge also fixes
the belief coordinates relative to the chosen task inventory: no
intrinsic meaning attaches to $\alpha_k$ compared across different
$K$. Under this fixed gauge the OU runs jointly settle at a large fitted
$\alpha$, a contracted belief term $\psi^\top\Sigma\,\psi$, and
variance carried by $\sigma_\text{obs}$---the empirical face of the
belief pruning; given the split non-identification
(\emph{Stability} below), no single causal direction is claimed.

\paragraph{Selecting the $\lambda$ grid fairly.}
The natural criterion for the grid---the unweighted validation sum
$\mathrm{MSE}+\mathrm{NLL}_1+\mathrm{NLL}_2$---is in mixed units.
Targets are centred but never divided by their scale, so on a corpus
whose state has spread $s$ the regression term enters at $\approx s^2$
against two log-losses pinned near unity, and the argmin is chosen
almost entirely by the regression task. Dividing the targets by their
training sd would remove that particular mismatch, and it is worth
being explicit about why we do not treat this as the answer:
standardisation is the \emph{manual, one-shot} version of what
$\sigma_\text{obs}^2$ does automatically. It fixes a single global
constant chosen before training, per task, from the training split; the
fitted noise scale tracks $c^2$ continuously at no search cost, and the
belief term tracks what no global constant can---scale that varies with
the input, which is the regime the rest of this paper is about. The
weighted sum needs the preprocessing step because it has no parameter
to absorb the units; composition has one. That is exactly the
mis-scaling defect Proposition~\ref{prop:scale} removes, so leaving it
inside our own baseline would score the baseline with the flaw we
claim to remove. Both grid rows of Table~\ref{tab:main_sde} therefore
re-select the \emph{same} $49$ runs on a common currency, scoring the
regression head as a Gaussian log likelihood at the predictive sd the
arm itself reports, so all three tasks enter in nats; the procedure is
bit-identical to the mixed-unit one when the criterion is not
switched. It changes almost nothing on OU or PM$_{2.5}$ ($\le0.024$
nats of summed test NLL; the two criteria select the identical run at
$38$ of those $40$ seed-by-head cells). On Lorenz, whose state spread
is far from one, they agree at $2$ of $20$, and the switch is worth
$0.150$ (CE) and $0.137$ (probit) nats to the baseline, moving the
ordinal head from $1.063$ to $0.907$ and from $1.218$ to $1.075$ at a
$\approx0.01$ nat cost to regression. The fair rows are the ones
reported throughout; since re-selection is a post-hoc argmin over an
identical sweep, their cost column reports that sweep's wall-clock as
measured under this table's conditions.

\paragraph{Stability.}
The variance split is not identified at finite $N$ and the objective is
not bounded below in $\sigma_\text{obs}$
(Appendix~\ref{app:scroll}), so the short-lead regime, where
$V^\star(\Delta)\to0$, is the place to look for collapse. We
observed none: no run was excluded, no divergence was seen, and across
every SCROLL-MT run reported here the smallest fitted
$\sigma_\text{obs}$ is $0.070$---on Lorenz at $\Delta{=}0.1$, where
$\sqrt{V^\star}$ is itself small---against $0.19$ at the short-lead OU
limit. Nothing floors $\sigma_\text{obs}$ itself, and no single
safeguard does: what keeps the degeneracy benign is the belief's
numerical floor $\varepsilon{=}10^{-4}$ on $\Sigma$, the
representation, and held-out scoring together. The floor contributes
$\varepsilon\|\psi_n\|^2$ to every plate variance, so it bounds $V_n$
away from zero only where the representation bounds $\|\psi_n\|$ away
from zero---layer norm fixes $\|\psi_n\|^2$ at $H$, an unnormalised
backbone does not---and held-out scoring is what penalises a
collapsing split rather than rewarding it.

\section{Detailed results}
\label{app:results}


\paragraph{Testbed properties, in full.}
\emph{(i) The likelihood family.} On the OU barrier event the arms split by
observation model alone. Every cross-entropy arm reaches $67\%$, below
the $80\%$ majority rate: the backbone and CE head are bias-free, so
the boundary is pinned at $x{=}0$, whose sign-agreement rate is
$67.5\%$; restoring a single bias term lifts a standalone CE head to
$89\%$. The threshold likelihood needs no such repair---its cutpoint
\emph{is} the offset---and every arm carrying it recovers the event at
$\approx0.27$ nats, MAP-probit at $\lambda{=}1$ and Kendall-probit
included, because SCROLL's observation models were ported into their
losses. Nor does any fixed family win across tasks: on the Lorenz
ordinal task the probit arms are \emph{worse} than CE, while SCROLL's
ordinal likelihood stays competitive (Diag) or wins outright (Full). A
per-task choice is unavoidable, and making it is composition's native
interface.

\emph{(ii) The loss weight.} The inert regime is the larger one: on OU and
PM$_{2.5}$ the $49$-run grid never beats its own $\lambda{=}1$ start
by $0.01$ nats, and on PM$_{2.5}$ it is up to $0.04$ nats
\emph{worse}. On Lorenz it collects $0.25$ (CE) and $0.18$
(probit) nats, but only once \emph{selected} in nats---see below,
where standardising the targets leaves that purchase
intact. Cost separates the arms far more
than the weights do ($13$--$18$\,s against $156$--$471$\,s per seed,
Table~\ref{tab:full_sde}), and the single run tunes \emph{more}: by
Corollary~\ref{cor:weights} it fits Gaussian noise scales, probit cutpoints and
$\alpha_k$ in the same pass, where the grid tunes the weights alone.
Kendall's learned weighting---also a single run---gives the
\emph{worst} Lorenz regression NLL at four times SCROLL-MT's
calibration error, so what the belief buys is not cheaper weighting but
input-dependent variance delivered \emph{stably}.

\emph{(iii) The event head's scope.} The event task separates arms only on
OU. On Lorenz the lobe label is a deterministic function of the state
at $\Delta{=}0.5$, so every arm that forecasts the state well
classifies it well (all within $0.166$--$0.189$ nats), and on
PM$_{2.5}$ the event head ties on the headline split and across the
rolling origins alike.

\paragraph{Full results.}
Table~\ref{tab:full_sde} reports all metrics behind
Table~\ref{tab:main_sde}: regression calibration error, event-task
accuracy, and wall-clock time per seed.

\begin{table}[t]
\centering\footnotesize
\setlength{\tabcolsep}{4.5pt}\renewcommand{\arraystretch}{0.92}
\caption{All metrics of the 3-task comparison (10 seeds,
protocol as in Table~\ref{tab:main_sde}).
Time = wall-clock per seed incl.\ tuning.
\textbf{Bold} = best per column within a system; \textit{italic} =
not significantly worse (one-sided paired $t$-test, $p\ge0.05$).}
\label{tab:full_sde}
\begin{tabular}{lrrrrrr}
\toprule
 & \multicolumn{2}{c}{regression} & \multicolumn{2}{c}{event (binary)} & \multicolumn{1}{c}{ordinal} & \\
 & NLL$\downarrow$ & CalErr$\downarrow$ & NLL$\downarrow$ & acc\,\%$\uparrow$ & NLL$\downarrow$ & time [s] \\
\midrule
\multicolumn{7}{l}{\emph{OU}} \\
SCROLL-MT (Diag) & 1.204 & 0.008 & 0.271 & \textit{89.3} & 0.881 & 13 \\
SCROLL-MT (Full) & 1.205 & \textit{0.008} & 0.265 & \textit{89.3} & 0.881 & 37 \\
SCROLL-MT (None, no belief) & 1.204 & \textit{0.009} & 0.271 & 89.2 & \textit{0.880} & 11 \\
\addlinespace[2pt]
MAP, CE heads, $\lambda{=}1$ & \textit{1.203} & \textit{0.009} & 0.417 & 67.1 & 0.954 & 5.3 \\
MAP, CE heads, $\lambda$ grid & 1.203 & 0.009 & 0.417 & 67.1 & 0.954 & 156 \\
Kendall, CE heads & \textbf{1.202} & \textit{0.008} & 0.417 & 67.1 & 0.954 & 5.6 \\
MLE $\sigma(x)$, CE heads & \textit{1.203} & \textit{0.008} & 0.417 & 67.1 & 0.954 & 4.3 \\
MAP, probit heads, $\lambda{=}1$ & \textit{1.203} & \textit{0.009} & 0.270 & 89.2 & \textit{0.880} & 8.5 \\
MAP, probit heads, $\lambda$ grid & \textit{1.203} & 0.009 & \textit{0.263} & \textit{89.3} & \textit{0.879} & 435 \\
Kendall, probit heads & 1.205 & \textit{0.008} & 0.266 & \textbf{89.4} & 0.884 & 8.0 \\
MLE $\sigma(x)$, probit heads & \textit{1.203} & \textbf{0.007} & 0.269 & 89.3 & \textit{0.880} & 8.6 \\
Laplace (Diag), probit heads & \textit{1.203} & \textit{0.009} & 0.270 & 89.2 & \textit{0.880} & 5.5 \\
Laplace (Full), probit heads & \textit{1.203} & \textit{0.009} & 0.270 & 89.2 & \textit{0.880} & 5.5 \\
SCROLL-MT (Diag), Ens-$5$ & 1.204 & 0.008 & 0.271 & \textit{89.3} & \textit{0.881} & 45 \\
MAP, CE heads, Ens-$5$ & \textit{1.202} & \textit{0.009} & 0.417 & 67.1 & 0.954 & 17 \\
MAP, probit heads, Ens-$5$ & \textit{1.203} & 0.009 & 0.270 & 89.2 & \textit{0.880} & 30 \\
\midrule
\multicolumn{7}{l}{\emph{stochastic Lorenz-63}} \\
SCROLL-MT (Diag) & 2.333 & \textit{0.016} & 0.170 & 92.5 & 1.032 & 18 \\
SCROLL-MT (Full) & \textbf{2.230} & \textbf{0.015} & 0.169 & \textit{92.6} & \textbf{0.843} & 29 \\
SCROLL-MT (None, no belief) & 2.596 & 0.098 & 0.168 & 92.3 & 1.069 & 11 \\
\addlinespace[2pt]
MAP, CE heads, $\lambda{=}1$ & 2.613 & 0.101 & 0.174 & \textit{92.7} & 1.154 & 7.5 \\
MAP, CE heads, $\lambda$ grid & 2.621 & 0.100 & \textbf{0.166} & 92.6 & 0.907 & 356 \\
Kendall, CE heads & 2.798 & 0.067 & 0.180 & 92.0 & 1.175 & 6.8 \\
MLE $\sigma(x)$, CE heads & 2.592 & \textit{0.017} & 0.182 & 91.8 & 1.128 & 8.6 \\
MAP, probit heads, $\lambda{=}1$ & 2.613 & 0.103 & 0.175 & \textit{92.7} & 1.252 & 10 \\
MAP, probit heads, $\lambda$ grid & 2.620 & 0.103 & \textit{0.169} & 92.6 & 1.075 & 471 \\
Kendall, probit heads & 2.810 & 0.065 & 0.180 & 92.0 & 1.249 & 9.6 \\
MLE $\sigma(x)$, probit heads & 2.612 & \textit{0.016} & 0.189 & 91.7 & 1.245 & 12 \\
Laplace (Diag), probit heads & 2.613 & 0.103 & 0.175 & \textit{92.7} & 1.253 & 9.1 \\
Laplace (Full), probit heads & 2.612 & 0.104 & 0.175 & \textit{92.7} & 1.251 & 8.8 \\
SCROLL-MT (Diag), Ens-$5$ & 2.295 & 0.033 & \textit{0.167} & \textit{92.7} & 1.004 & 90 \\
MAP, CE heads, Ens-$5$ & 2.608 & 0.104 & 0.172 & \textit{92.8} & 1.138 & 34 \\
MAP, probit heads, Ens-$5$ & 2.607 & 0.106 & 0.173 & \textbf{92.8} & 1.244 & 47 \\
\midrule
\multicolumn{7}{l}{\emph{Beijing PM$_{2.5}$ (real data, 24\,h lead)}} \\
SCROLL-MT (Diag) & 1.250 & \textit{0.024} & 0.450 & 78.5 & 1.203 & 0.8 \\
SCROLL-MT (Full) & \textit{1.255} & \textbf{0.020} & 0.489 & 77.7 & 1.212 & 1.3 \\
SCROLL-MT (None, no belief) & 1.311 & \textit{0.021} & 0.451 & 78.2 & 1.291 & 0.3 \\
\addlinespace[2pt]
MAP, CE heads, $\lambda{=}1$ & 1.312 & 0.026 & 0.451 & 78.0 & 1.229 & 0.2 \\
MAP, CE heads, $\lambda$ grid & 1.318 & \textit{0.023} & 0.456 & 77.6 & 1.229 & 10 \\
Kendall, CE heads & 1.312 & 0.029 & 0.458 & 78.4 & 1.238 & 0.2 \\
MLE $\sigma(x)$, CE heads & 1.305 & \textit{0.021} & 0.453 & 78.2 & 1.230 & 0.3 \\
MAP, probit heads, $\lambda{=}1$ & 1.312 & 0.030 & 0.452 & 78.2 & 1.288 & 0.3 \\
MAP, probit heads, $\lambda$ grid & 1.334 & \textit{0.023} & 0.460 & 77.8 & 1.297 & 15 \\
Kendall, probit heads & 1.312 & 0.031 & \textit{0.440} & 77.9 & 1.291 & 0.3 \\
MLE $\sigma(x)$, probit heads & \textit{3.011} & \textit{0.026} & 0.458 & 77.8 & 1.296 & 0.4 \\
Laplace (Diag), probit heads & 1.306 & 0.035 & 0.450 & 78.2 & 1.287 & 0.3 \\
Laplace (Full), probit heads & 1.310 & 0.030 & 0.448 & 78.2 & 1.288 & 0.3 \\
SCROLL-MT (Diag), Ens-$5$ & \textbf{1.243} & \textit{0.022} & \textit{0.443} & \textbf{79.3} & \textbf{1.197} & 3.9 \\
MAP, CE heads, Ens-$5$ & 1.303 & 0.030 & \textbf{0.440} & 77.9 & 1.215 & 1.0 \\
MAP, probit heads, Ens-$5$ & 1.308 & 0.029 & \textit{0.445} & 77.9 & 1.279 & 1.4 \\
\bottomrule
\end{tabular}
\end{table}

\paragraph{Why the headline comparison is a sum, and what it hides.}
The per-corpus totals quoted in Section~\ref{sec:experiments} sum the
three heads' test NLL. That is not a weighting we chose: \eqref{eq:fmt}
\emph{is} an unweighted sum of proper log scores, so scoring by the
unweighted test sum evaluates the estimand actually minimised, and
Corollary~\ref{cor:inert}'s caveat---that no scale-free principle picks
a point on a trade-off surface---concerns whether that point matches a
particular user's utility, which is a separate question from whether we
report our own objective consistently. The distinction matters most
where a trade-off is real, and on Lorenz it is: the fairly selected CE
grid is preferable to SCROLL-MT (Diag) on both discrete heads ($0.166$
vs $0.170$ on the event, $0.907$ vs $1.032$ on the regime) and loses
only the state head, by $0.29$ nats. The summed margin is therefore the
objective's own point on that surface, not a dominance claim---and the
arm that does dominate the grid there is SCROLL-Full, which wins the
state and regime heads and ties the event.

\paragraph{Standardising the targets.}
The mixed-units selection argument (\emph{Selecting the $\lambda$ grid
fairly}, Appendix~\ref{app:details}) is worth settling by measurement
rather than principle,
so we reran the Lorenz corpus with the regression target divided by its
training sd $s$ and nothing else changed. Standardisation shifts the
regression NLL by the constant $\log s$, leaving every \emph{difference}
below comparable with the numbers above, and it separates two claims
that the mixed-unit criterion had confounded.

The criterion swap is the half that was units. Once the targets are
standardised the two criteria select the same run on $9$ of $10$ seeds
on both heads---against $2$ of $20$ before---and the swap is worth
$0.005\pm0.005$ and $-0.000\pm0.000$ nats rather than $0.150$ and
$0.137$: re-selecting in nats was correcting a scale mismatch, exactly
as claimed, and corrects nothing once the mismatch is gone.

The grid's purchase over its own $\lambda{=}1$ start is the half that
was not. It survives standardisation at $0.291\pm0.015$ (CE) and
$0.174\pm0.010$ (probit), so it is not an artefact of our
preprocessing, and we do not claim it as one. What the target scale
decides is not \emph{whether} the weighted sum needs tuning but
\emph{where} it must look: unstandardised, the selected
$\lambda_1$ spreads over $[0.01,10]$; standardised, the argmin moves
onto the grid boundary $\lambda{=}(30,30)$ on $9$ of $10$ CE seeds,
where the sweep is truncated and the optimum may lie outside it. The
composed objective never looks. Written in the units of the weighted
sum its effective squared-error weight is $(2V_n)^{-1}$, and it falls
by $61\times$ against
a target-variance ratio of $63\times$---Proposition~\ref{prop:scale}'s
map, fitted in the same pass rather than searched over $49$ runs.

Nothing the paper claims for SCROLL-MT weakens on the standardised
corpus and its margin over the better fair grid is if anything larger
(the two corpora are not paired, so we claim only that it does not
shrink): $-0.215\pm0.023$ nats for Diag and $-0.505\pm0.021$ for Full,
against $-0.159$ and $-0.453$ before, at a $26$--$37\times$ cost ratio
and at a calibration error of $0.020$ against the grid's $0.084$. (The
seconds behind that ratio are larger than Table~\ref{tab:full_sde}'s
because this probe ran its arms concurrently; only the ratio, measured
under one set of conditions, is comparable across the two.) The arm standardisation helps most is the one with
least to fit: the no-belief None arm, which trails the CE grid by
$0.139$ nats unstandardised, ties it ($-0.017\pm0.017$).

\paragraph{The margin across target scales.}
Standardisation is one point on a scale axis, so we swept the axis:
Lorenz targets multiplied by $c\in\{\tfrac18,\tfrac12,1,2,8\}$,
everything else unchanged, at two fixed backbones. On tanh+LN---the
backbone SCROLL-MT (Diag) actually selects---the Diag margin over the
fair CE grid holds at every scale, $-0.19$ to $-0.31$ nats of summed
test NLL across the $64\times$ range, and Full holds on \emph{both}
backbones at every scale ($-0.41$ to $-0.51$). On the unnormalised
tanh backbone Diag never separates from the grid at any scale
($-0.00$ to $+0.14$): the representation limit of
Table~\ref{tab:routing}, present at every $c$, not a scale effect.
What drift the margins do show toward the extreme scales is the
optimiser's, not the objective's:

\paragraph{Equivariance in practice: objective versus optimiser.}
Proposition~\ref{prop:scale} is a statement about the minimiser; a
fixed-budget run need not reach it. Rescaling one task by $c{=}8$ and
comparing summed test NLL against the $c{=}1$ run plus the
proposition's constant leaves a gap of $0.216\pm0.013$ nats ($10$
seeds, s.e.; $0.074$ at $c{=}\tfrac18$). Warm-starting the $c{=}8$ run
at the $c{=}1$ solution pushed through the proposition's parameter map
closes the gap exactly ($0.000\pm0.000$): the mapped
objective--solution pair is equivariant to machine precision in this
tested range, and the entire gap is optimisation. Attribution
by arm agrees: mapping only the initialisation removes $39\%$ of it, a
$4\times$ step budget $24\%$, the two together $76\%$, while the
numerical floor $\varepsilon$ on $\Sigma$ contributes nothing
($0\%$)---from a scale-free start the run simply ends far from the
mapped optimum (the fitted $\alpha$ lands orders of magnitude off the
mapped value at $c{=}8$). The practical reading: the objective needs
no per-scale retuning, and a mapped warm start recovers whatever a
fixed budget leaves behind.

\paragraph{What equivariance does not buy.}
The reach of a \emph{search} is not equivariant even though the
optimum is: the compensating weight grows as $c^2$, so a fixed grid is
outrun by a scale mismatch it was not sized for---on Lorenz its argmin
sits on the boundary of the sweep---while $\sigma_\text{obs}^2$ tracks
$c^2$ continuously and at no search cost, by a measured $61\times$
against a variance ratio of $63\times$. That governs where the search
must look, not who wins: standardising the targets leaves SCROLL-MT's
margin intact (above), so the separation in
Section~\ref{sec:experiments} is not the scale mismatch. Two
implementation caveats. Proposition~\ref{prop:scale} concerns the
ideal covariance families: the implemented $\varepsilon I$ floor on
$\Sigma$ could in principle bind after the map and break exactness at
the boundary---in the tested $64\times$ range it does not, the
attribution probe's $\varepsilon$ arm contributing $0\%$ of the
observed gap. And the minimiser's
equivariance is exact while a fixed-budget optimiser's is not; the
previous paragraph measures that gap and attributes all of it
to optimisation---a mapped warm start closes it exactly.

\paragraph{Kernel recovery on the well-specified anchor.}
On OU every arm attains the entropy floor
$\tfrac12\log(2\pi e\,V^\star(\Delta))$ within $1.5\%$ for
$\Delta\ge0.25$ and $4.5\%$ at $\Delta{=}0.1$, so the anchor separates
nothing by NLL---which is the point of running it. What it does test is
the belief: \eqref{eq:fmt} places no cap on $\psi^\top\Sigma\,\psi$, and
at $H{=}50$ against a scalar state there is ample room for spurious
spread, yet the fitted predictive sd stays within $2\%$ of
$\sqrt{V^\star}$ at every lead time and $\sigma_\text{obs}$ within
$4\%$, the runs jointly settling at a large fitted $\alpha$ and a
contracted belief term
(median fitted $\alpha$ on the regression head: $379$, against $0.11$
on Lorenz; the gauge reading of Appendix~\ref{app:details}). A homoscedastic system is where a freely-routed variance
could most easily invent structure, and it does not.

\paragraph{The closed-routing arm and the binding-map-only cell.}
The closed-routing arm replaces the free belief of the regression head
with the conjugate posterior at the current backbone,
$\Sigma=(\alpha I+\sigma_\text{obs}^{-2}\Psi^\top\Psi)^{-1}$
recomputed from the training design matrix at every step, with that
head's data and prior terms replaced by the evidence of the same
conjugate model---the corner whose predictive variance is the
residual-independent leverage variance (Appendix~\ref{app:scroll}).
The other two heads' terms are untouched, so the backbone is trained
throughout by their shared-cavity terms plus the regression head's evidence.
Only the regression head is modified: the
two ordinal heads stay freely routed in every arm, no closed form
being available for them, so the head inventory is fixed. In this
headline closed arm two things change together in the modified head,
and only one of them is the route. Constraining $(\mu,\Sigma)$ to the
binding map while keeping the shared-cavity data term isolates
routing alone; substituting the conjugate model's evidence also changes
the estimand, from a composite predictive score to a marginal
likelihood. Running that binding-map-only cell separates the two: it
reproduces the evidence corner to within $0.014$ nats on every Lorenz
backbone (within $0.006$ on OU), so the pinned flatness of $V_n$ is
the binding map's doing, and the estimand swap adds nothing on these
corpora. On OU the five regression routes---None, closed,
binding-map-only, Diag and Full---tie
($1.202$--$1.203$ test NLL, predictive sd $0.809$--$0.812$ against
$\sqrt{V^\star}=0.795$), which is what a homoscedastic system entitles
the closed route to; the tracking separation is confined to the
heteroscedastic system, as the theory predicts.

Correlations are Pearson on the raw $V_n$ against $V^\star(x_n)$, and
because $V^\star$ spans $284\times$ they are influenced by the largest
states; the rank correlations are $0.41$ (closed), $0.56$ (Diag) and
$0.65$ (Full). The ordering is unchanged but the free-versus-closed gap
is much narrower on ranks ($0.56$ vs $0.41$) than on raw values ($0.61$
vs $0.17$), so the assumption-free form of the routing claim is not the
correlation at all but the range: the closed route's $V_n$ moves under
$1.3\times$ while $V^\star$ moves $284\times$. Against the residual
variance about each arm's own mean, Full sits within $3\%$ while Diag
is $12\%$ \emph{under}---the same ordering as the correlations, and a
further reason the paper reports both families rather than one.

\paragraph{Routing across backbones.}
The per-state probe fixes one backbone for every arm, which makes its
comparison matched but not automatically representative, so
Table~\ref{tab:routing} reruns every arm at all four. Two effects
separate there. The pinned route is flat wherever it is run: $V_n$
spans at most $1.26\times$ against a $284\times$ range in $V^\star$,
and never correlates with it ($|r|\le0.20$), so removing the binding is
what frees $V_n$ to vary at all, independently of the
representation---and of the estimand: the binding-map-only cell above
reproduces the corner on every backbone.
Whether that freedom is worth anything then depends on what the freed
family can express. Full improves on the pinned route on every
backbone, by $0.137$ to $0.314$ nats; Diag improves on it under layer
norm ($0.129$, $0.237$) and ties without ($-0.022\pm0.006$,
$-0.003\pm0.011$), its $V_n$ range falling from $46\times$ to
$16\times$ and its correlation from $0.61$ to $0.15$. A diagonal belief
on an unnormalised representation cannot reach $V^\star$; that is a
limit on the family, not on the routing, and it is why the two are
worth separating. On OU every route ties on every backbone (within
$0.007$ nats), and the belief adds no spread of its own: the freely
routed total sd stays within $0.02$ of the no-belief arm's own ratio to
$\sqrt{V^\star}$ on all four, that ratio being $1.02$ on three and
$1.06$ on relu$+$LN, where every arm degrades together. The anchor
carries no backbone exposure at all.

\begin{table}[t]
\centering\footnotesize
\setlength{\tabcolsep}{4.5pt}\renewcommand{\arraystretch}{0.92}
\caption{Routing across backbones on Lorenz-63 at $\Delta{=}0.5$
(10 seeds): test NLL and Pearson $r(V_n,V^\star(x_n))$ for
each belief family against the closed (evidence-pinned) route, at
every backbone rather than the single one the per-state probe fixes.
The arms of Table~\ref{tab:main_sde} select tanh$+$LN (Diag) and
tanh (Full), both $10/10$ by validation loss.}
\label{tab:routing}
\begin{tabular}{lrrrrrrr}
\toprule
 & \multicolumn{1}{c}{None} & \multicolumn{2}{c}{closed} & \multicolumn{2}{c}{Diag} & \multicolumn{2}{c}{Full} \\
\cmidrule(lr){2-2} \cmidrule(lr){3-4} \cmidrule(lr){5-6} \cmidrule(lr){7-8}
 & NLL$\downarrow$ & NLL$\downarrow$ & $r$ & NLL$\downarrow$ & $r$ & NLL$\downarrow$ & $r$ \\
\midrule
relu & 3.189 & 3.169 & -0.05 & 3.191 & +0.09 & 3.032 & +0.06 \\
relu$+$LN & 3.151 & 3.055 & +0.20 & 2.926 & +0.43 & 2.843 & +0.43 \\
tanh & 2.597 & 2.545 & +0.09 & 2.548 & +0.15 & 2.238 & +0.85 \\
tanh$+$LN & 2.620 & 2.569 & +0.17 & 2.333 & +0.61 & 2.256 & +0.81 \\
\bottomrule
\end{tabular}
\end{table}

\paragraph{Normalisation interaction under layer norm.}
On Lorenz with a tanh+LN backbone, SCROLL-MT (Diag) already reaches its
headline quality ($2.33$, $0.016$; Full improves further,
Table~\ref{tab:full_sde}) while \emph{no} MSE-trained MAP arm improves
under LN (all $\ge(2.61, 0.09)$), Kendall recovers part of the
calibration ($0.065$) but none of the NLL ($2.79$), and the NLL-trained
MLE $\sigma(x)$ arm recovers the calibration ($0.012$) but not the NLL
($2.59$): exploiting a normalised representation takes NLL training
\emph{and} the input-dependent belief.

\paragraph{Which covariance family, and why Diag is the default.}
The two families split by corpus rather than dominating one another.
Full is decisively better on Lorenz ($-0.104$ nats on the state head,
$-0.189$ on the regime, and it is the only arm to dominate the fairly
selected grid there rather than trade against it), and it tracks
$V^\star(x)$ far more closely ($r{=}0.81$ against $0.61$). On
PM$_{2.5}$ it is worse on every head, most on the event ($+0.038$), and
the same ordering appears across the rolling origins below
($+0.023\pm0.007$). The reading we take is that a full $H\times H$
belief at $N_\text{tr}\approx10^3$ is under-determined, and that a
probit head---having no $\sigma_\text{obs}$ to absorb surplus
spread---is where that shows first, the same asymmetry that makes an
ordinal head sensitive to an inflated belief elsewhere in this
appendix. Diag is therefore the default: it wins the corpus with real
data and least of it, at a fraction of Full's cost. Where the belief
has the most room, Full is the better arm, and the table reports both.

\paragraph{An approximate posterior selector: last-layer Laplace.}
A post-hoc Laplace approximation at the trained MAP point supplies a
standard approximate-posterior comparator: train the point estimate,
set each head's belief to the Gaussian its curvature prescribes (prior
precisions validation-selected per head), and predict through the same
plate variance. Architecture, protocol, covariance family, and
predictive form are matched to SCROLL-MT and the selector of $\Sigma$
differs---the objective there, the posterior here---but so do the
backbone and the mean, trained under the MAP loss rather than the
composite score, so this pair does not isolate routing as cleanly as
the binding-map-only cell above; its role is to show the result
survives a widely deployed posterior selector. Paired per backbone on
the state head, the ordering is the routing table's: on OU the two tie
($|\Delta\text{NLL}|\le0.03$ on every backbone and family); on Lorenz
the objective-selected belief wins on all four backbones, by
$0.03$--$0.29$ nats (Diag) and $0.19$--$0.38$ (Full); on PM$_{2.5}$
Full wins on all four ($0.03$--$0.07$) while Diag is mixed ($-0.05$ to
$+0.04$, behind only on tanh+LN)---the same family limit as above.
Table~\ref{tab:full_sde} carries the validation-selected rows.

\paragraph{Deep-ensemble arms.}
Member $k$ of an Ens-$5$ row is the named arm retrained at seed
$s+1000k$, so the five members differ in initialisation and in nothing
else. On the Gaussian head the mixture predictive is collapsed to its
first two moments \citep{lakshminarayanan2017simple},
$\bar\mu=\operatorname{mean}_k\mu_k$ and
$\bar V=\operatorname{mean}_k(V_k+\mu_k^2)-\bar\mu^2$; on the ordinal
and event heads the members' category probabilities are averaged and
scored as one predictive distribution. Both sides of the table are
ensembled, because a single-pass method against a five-member ensemble
is not a like-for-like comparison: the table therefore carries the
cost-mismatched contrast ($1$ run against $5$) and the cost-matched one
($5$ against $5$). Two caveats. Ensembling is applied at
$\lambda{=}1$ rather than to the grid arms, which would cost $49\times5$
runs and buy little here since the weight is inert on PM$_{2.5}$; and
the ten Ens-$5$ replicates at seeds $s{=}5\dots14$ draw members from an
overlapping pool, so their error bars are not independent of the
single-run rows above them and the two should not be differenced as if
they were.

\paragraph{Sharing and interference.}
The single-task reference trains each head \emph{alone} on a backbone of
its own, with the loss terms of \eqref{eq:fmt} unchanged---the Gaussian
head with its diagonal belief, the ordinal heads with theirs---so the
only difference from the joint arm is which tasks are active, not the
head inventory, the observation models, or the optimiser. $K$
independent models cost $K$ backbones; the joint arm costs one. Writing
$\text{interference}_k = \text{NLL}_k(\text{joint}) -
\text{NLL}_k(\text{alone})$, positive meaning sharing costs task $k$
($10$ seeds, backbone validation-selected per head as everywhere else):

\begin{center}\footnotesize
\begin{tabular}{lrrr}
\toprule
 & state & event & regime \\
\midrule
OU & $+0.000\pm0.001$ & $+0.004\pm0.001$ & $+0.001\pm0.001$ \\
Lorenz-63 & $+0.057\pm0.009$ & $+0.003\pm0.001$ & $+0.072\pm0.024$ \\
PM$_{2.5}$ (real) & $-0.006\pm0.007$ & $-0.018\pm0.006$ & $-0.014\pm0.004$ \\
\bottomrule
\end{tabular}
\end{center}

\noindent The three regimes track how much one feature map can serve at
once, the same axis Corollary~\ref{cor:inert} turns on. On OU the three
observables are functionals of one scalar state and sharing is free. On
Lorenz they compete and it costs. On the real series---$N{=}1{,}680$
against $D{=}36$ inputs, with the three observables read off the same
future concentration---the heads inform one another and sharing pays on
both discrete tasks. The joint arm of this probe is SCROLL-MT (Diag)
itself, so the PM$_{2.5}$ row is a statement about the model the rest of
the paper reports.

\paragraph{Does the real arm survive re-splitting?}
The headline PM$_{2.5}$ split is chronological and its ten seeds vary
initialisation only, so every $\pm$ on that corpus measures
initialisation spread, not generalisation. Five rolling origins move
the split instead: origin $j$ trains on $[0,(0.4+0.1j)N)$ with $10\%$
validation and test blocks immediately after, so the evaluation window
walks forward through the series. Uncertainty is then reported
\emph{between} origins ($n{=}5$), since seeds inside one origin share
its split---pooling all $50$ runs would treat a single split as five
independent ones and roughly halve every interval. Pooled that way,
SCROLL-MT (Diag)'s advantage on the state head holds against every
weighting baseline at $-0.061$ to $-0.088$ nats, every one of them at
$4.8$ between-origin standard errors or more. On the regime head it holds against the MAP arms
and Kendall-probit ($-0.029$ to $-0.081$) but not against Kendall-CE
($-0.020\pm0.014$), the one weighting baseline it does not separate
from there. The event head ties, as it does on the headline split.
Sharing (above) is the quantity the origins do \emph{not} resolve: the
point estimate favours the joint model on all three heads, but the
per-origin sign holds at only three to four of five and every
between-origin interval covers zero ($-0.028\pm0.016$,
$-0.010\pm0.007$, $-0.013\pm0.007$).

\paragraph{The frozen-$\alpha$ ablation.}
The frozen-$\alpha$ arms rerun SCROLL-MT (Diag) with all three prior
precisions fixed at a shared
$\alpha\in\{0.01,0.1,1,10,10^2,10^3,10^4\}$ and excluded from the
optimiser; everything else (backbone sweep, seeds, early stopping) is
identical, and the self-tuned row of Table~\ref{tab:alpha_sde} is the unchanged
SCROLL-MT (Diag) arm of Table~\ref{tab:main_sde}. A per-head grid at
the same resolution would already take $7^3=343$ runs. Three
details. \textbf{(i)}~On OU every $\alpha$ ties because
the \emph{total} predictive sd stays within $2.3\%$ of $\sqrt{V^\star}$
at every $\alpha$: score-optimality pins the total, the prior only
routes it between $\sigma_\text{obs}$ and the belief, and on a
homoscedastic system any split scores the same. Sensitivity is
confined to where the belief carries input-dependent spread---exactly
the cells SCROLL-MT wins ($+0.12$ nats on the Lorenz ordinal task at
$\alpha\ge100$, alongside $+0.05$ nats of regression NLL and
twice the calibration error).
\textbf{(ii)}~The loop's one significant purchase is $0.01$ of average
coverage on PM$_{2.5}$ (the one-parameter property noted in
Section~\ref{sec:experiments});
there the validation-selected $\alpha$ scatters across the full grid,
a flat landscape where the loop is wasted compute.
\textbf{(iii)}~A scope note: a single fixed $\alpha{=}0.01$ is also
within $0.01$ nats of the fitted run on all three corpora---a verdict
they deliver only in hindsight, at their $N$, and not a default one
could have set in advance. The implicit fit is studied extensively
in the single-task setting by \citet{prochazka2026bethe}; stress-testing
it across multi-task regimes beyond these corpora is an open empirical
gap (Appendix~\ref{app:future}).

\begin{table}[t]
\centering\footnotesize
\setlength{\tabcolsep}{4.5pt}\renewcommand{\arraystretch}{0.92}
\caption{Frozen-$\alpha$ ablation: SCROLL-MT (Diag) with all
three prior precisions frozen at a shared $\alpha$ vs the
precisions the run tunes itself (10 seeds, backbone
validation-selected per seed as in Table~\ref{tab:main_sde};
the grid row selects $\alpha$ per seed by validation loss).
\textbf{Bold}/\textit{italic} as in Table~\ref{tab:main_sde}.}
\label{tab:alpha_sde}
\begin{tabular}{lrrrrrrrrr}
\toprule
 & \multicolumn{3}{c}{OU} & \multicolumn{3}{c}{Lorenz-63} & \multicolumn{3}{c}{PM$_{2.5}$ (real)} \\
\cmidrule(lr){2-4} \cmidrule(lr){5-7} \cmidrule(lr){8-10}
 & NLL$\downarrow$ & CalErr$\downarrow$ & ord$\downarrow$ & NLL$\downarrow$ & CalErr$\downarrow$ & ord$\downarrow$ & NLL$\downarrow$ & CalErr$\downarrow$ & ord$\downarrow$ \\
\midrule
self-tuned $\alpha_k$ ($1$ run) & \textit{1.204} & 0.008 & \textit{0.881} & \textit{2.333} & \textit{0.016} & \textit{1.032} & \textit{1.250} & 0.024 & 1.203 \\
$\alpha$ grid, val.-selected ($7$ runs) & 1.204 & \textit{0.008} & \textit{0.880} & \textbf{2.327} & \textit{0.015} & \textbf{1.015} & \textit{1.252} & \textbf{0.015} & \textit{1.201} \\
\addlinespace[2pt]
fixed $\alpha{=}0.01$ & \textbf{1.203} & \textit{0.008} & \textit{0.880} & \textit{2.328} & \textit{0.015} & \textit{1.018} & \textit{1.248} & \textit{0.018} & \textbf{1.199} \\
fixed $\alpha{=}0.1$ & \textit{1.204} & 0.008 & \textit{0.881} & \textit{2.329} & \textit{0.015} & \textit{1.018} & \textit{1.254} & \textit{0.020} & \textit{1.200} \\
fixed $\alpha{=}1$ & 1.205 & 0.008 & \textit{0.881} & 2.354 & \textbf{0.014} & 1.044 & \textit{1.251} & 0.023 & 1.205 \\
fixed $\alpha{=}10$ & 1.205 & \textit{0.008} & \textit{0.880} & 2.387 & 0.019 & 1.092 & \textbf{1.248} & 0.026 & 1.205 \\
fixed $\alpha{=}100$ & \textit{1.204} & \textbf{0.007} & \textbf{0.880} & 2.389 & 0.031 & 1.157 & \textit{1.249} & 0.024 & 1.206 \\
fixed $\alpha{=}10^{3}$ & \textit{1.204} & \textit{0.008} & \textit{0.880} & 2.380 & 0.033 & 1.148 & \textit{1.248} & \textit{0.022} & \textit{1.203} \\
fixed $\alpha{=}10^{4}$ & 1.206 & \textit{0.008} & 0.882 & 2.384 & 0.033 & 1.149 & \textit{1.255} & 0.023 & 1.204 \\
\bottomrule
\end{tabular}
\end{table}

\section{Open directions}
\label{app:future}

A short validation study fixes its scope deliberately
(Section~\ref{sec:outlook}); this appendix records the doors it opens.

\paragraph{Scaling in the number of tasks.}
The comparison here sits at $K{=}3$, where the $49$-run $\lambda$ grid
is expensive but still feasible. The regime likelihood composition is
built for is $K{\ge}5$, where grids are out of the question and the
honest tuned baselines become budget-matched random search and learned
uncertainty weighting. The open question is whether the advantage grows
with $K$ as the $O(K)$ parameterisation predicts. One weighting
composition does not remove also grows with $K$: the task
\emph{inventory}. Listing the same observable twice doubles its data
terms, so ``no task weights'' is a statement at a fixed task list, and
the list prices tasks in integers; at $K{=}3$ the inventory is given,
but at the $K$ this paragraph contemplates choosing it becomes the
design decision---the one lever the units argument of
Proposition~\ref{prop:scale} says nothing about.

\paragraph{Jointly trained horizon sets.}
Lead times are separate runs here. Treating the horizons themselves as
composed tasks---one head per $\Delta$ on a shared backbone---would make
the $\sigma(\Delta)$ curve of Figure~\ref{fig:sigma_delta} the output of
a \emph{single} model and turn cross-horizon consistency (for OU, the
analytic monotone $V^\star(\Delta)$) into a checkable structural
property.

\paragraph{Interference in the shared backbone.}
Corollary~\ref{cor:mt_score} is stated at fixed $\psi$, and per-task
score-optimality need not survive joint training of $\psi$: one feature
map need not support all $K$ optima at once. The gap---each head under joint
training against the same head trained alone---is measured in
Appendix~\ref{app:results}, and it does not vanish: sharing costs
$0.06$--$0.07$ nats on Lorenz. Two mechanisms would produce that and
this design separates neither: a feature map of finite width unable to
serve three optima at once, and the optimisation itself, where
competing gradients settle on a compromise the same width might have
avoided. Widening $H$ at a fixed task inventory would tell them apart.
What sets the sign is open, and it is
the quantity a $K$-scaling study would have to track.

\paragraph{Per-task prior precisions.}
Two claims hide inside the prior-fit result and only the first is
settled here. The
implicit fit replaces the validation grid---one run where the reference
needs seven, and which fixed $\alpha$ would have been safe is knowable
only after running it. Whether the tasks want \emph{different}
$\alpha_k$ is a second question, and the reason it is open is
arithmetic: the per-head oracle grid costs $|A|^K$ runs, $343$ at
$K{=}3$ for the resolution used here. At $K{=}2$ it costs $49$, so we
probed it there, on a synthetic pair built to pull the two heads apart
(a Gaussian head, whose $\sigma_\text{obs}$ absorbs whatever spread the
belief does not, against an ordinal head that has no such absorber).
The preference is unambiguous and the payoff is not. Selecting
$(\alpha_1,\alpha_2)$ per seed by validation lands \emph{off} the
diagonal on $37$ of $40$ seeds, and the implicit fit separates the two
heads on every seed---$\alpha_\text{reg}\in[17.5,31.7]$ against
$\alpha_\text{ord}\in[0.72,1.34]$, a $\approx21\times$ gap, in the same
direction the oracle picks, and against a within-corpus spread never
above $6\times$ on the three corpora of Table~\ref{tab:main_sde}. Yet
the oracle's advantage over the best shared $\alpha$ is
$0.007\pm0.005$ and $0.003\pm0.004$ nats on the two heads: seven times
the compute, bounded by noise. The one-run fit sits within $0.015$
nats of both grids, marginally behind. So per-task precisions are a
structure the objective expresses consistently and, on this corpus,
does not get paid for---which leaves open whether a regime exists where
it is paid, and makes the $K$-scaling above the place to look, since
that is where the oracle a shared grid approximates becomes
unaffordable.

\paragraph{A joint observation model.}
For multiple Gaussian heads, composing the tasks as a product gives a
joint Gaussian whose noise
covariance is \emph{diagonal}, a strict subset of the joint likelihood: a
full $\Sigma_\text{obs}$ also carries the residual correlation left after
conditioning on the shared backbone. On these corpora that restriction is
not a mild approximation---the Lorenz event label is $\mathds{1}[x(t{+}\Delta)>0]$
and its regression target is $x(t{+}\Delta)$, so one observable is a
deterministic function of another, and the PM$_{2.5}$ band and state share
the same value. Nothing claimed here rests on the difference:
Corollary~\ref{cor:mt_score} is a per-task \emph{marginal} statement, and
each marginal reaches its best representable score optimum at fixed
$\psi$, irrespective of unmodelled cross-task dependence. What the
composition as written cannot represent is the \emph{joint} predictive---the
coherence between observables of the same window. Two regression heads with
correlated noise make the size of that gap measurable, and with independent
beliefs the coupled data term stays closed form, since the belief adds no
cross-head covariance:
$V_n=\Sigma_\text{obs}+\operatorname{diag}(\psi_n^\top\Sigma_1\psi_n,
\psi_n^\top\Sigma_2\psi_n)$. In a controlled probe over noise correlation
$\rho$ the full model recovers $\rho$ ($\hat\rho=0.80$ at $\rho=0.8$),
improves the joint test NLL by $0.163\pm0.009$ nats---about $80\%$ of the
$-\tfrac12\mathbb{E}\log(1-\rho_n^2)$ available given the predictive
covariance, where
$\rho_n=\Sigma_{\text{obs},12}/\sqrt{V_{1,n}V_{2,n}}$ is the plate's
effective predictive correlation---leaves both marginals unchanged, and at $\rho=0$ neither
invents correlation ($\hat\rho=-0.000$) nor pays for the extra parameter
($0.000\pm0.007$). This is also the sharper form of the paper's own
argument: scalar squared-error weights correspond to diagonal entries
of the noise \emph{precision}, and a weighted sum has no slot for an
off-diagonal term at any weights; a coupled likelihood carries the
full covariance---equivalently, the off-diagonal precision
interactions---that no choice of scalar task weights can express, and
Corollary~\ref{cor:weights} identifies the scalar weights with the
precisions of a covariance the composition currently constrains to be
diagonal. The construction extends the same way to a latent composed
task---marginalising a categorical head gives a mixture output, at the cost
of the per-factor closed form.

\paragraph{Baselines on other axes.}
Gradient-surgery methods \citep{chen2018gradnorm,yu2020pcgrad} modify
gradients rather than loss weights; they are orthogonal to likelihood
composition and could be composed with it rather than compared against
it. Deep ensembles and a post-hoc Laplace selector are run
(Appendix~\ref{app:results}); the deep-uncertainty baselines still
missing are a variational Bayesian last layer
\citep{harrison2024variational} and the heterogeneous multi-output GP
whose construction this head mirrors
\citep{moreno2018heterogeneous}---tractable at these $N$ and the
natural non-neural reference.

\paragraph{More real series, and more of them.}
The rolling origins of Appendix~\ref{app:results} take the real-data
claims off a single split, but five origins over one $1{,}680$-window
series is what that corpus can support, and they overlap heavily. The
quantity they leave unresolved---whether sharing helps or is merely
harmless---is the one that would benefit most from several independent
series rather than more origins or more seeds within this one.

\paragraph{Richer systems.}
Higher-dimensional dynamics (stochastic PDEs, larger attractors) where
the backbone must work harder, and further real series with externally
dictated thresholds (air quality, hydrology), to which the
PM$_{2.5}$ recipe of Appendix~\ref{app:details} transfers unchanged.

\paragraph{Theory.}
The data term identifies only the total predictive variance
$\sigma_\text{obs}^2+\psi^\top\Sigma\,\psi$; the split is selected by
the $O(1/N)$ prior term (Appendix~\ref{app:scroll}). Making the
finite-$N$ behaviour of that split precise, and extending the
entropy-floor argument of the OU anchor beyond well-specified
regression,
are natural next steps---alongside the programme below.

\paragraph{The sequential cavity and adaptive filtering.}
SCROLL's shared cavity is an exchangeable approximation to a
\emph{sequential} cavity, whose order-dependence artifact disappears
exactly when the data carry a canonical order---which temporal data do.
The sequential-cavity predictive is then the prequential one; in the
linear-Gaussian case the construction collapses to the Kalman filter,
and freely-routed variances become online, score-driven
noise-covariance adaptation---the classical territory of adaptive
filtering \citep{mehra1972approaches,sarkka2009vbakf}. Making free
routing well-defined along a single filtering pass (SCROLL's routing is
defined on an optimisation trajectory) is the open question we consider
most promising, with this paper's composed likelihoods as its natural
multi-observable emission model.

\end{document}